%% file: calm_paper_nodiff.tex
\pdfoutput=1
\documentclass[conference,letterpaper,10pt,nofonttune]{IEEEtran}
\IEEEoverridecommandlockouts
\usepackage[left=0.76in,right=0.76in,top=0.75in,bottom=0.81in,
            columnsep=0.2in,noheadfoot]{geometry}
\renewcommand{\IEEEtitletopspaceextra}{0.17in}
\makeatletter
\def\@IEEEauthorblockNstyle{\normalfont\fontsize{11}{13}\selectfont}
\makeatother
\IEEEaftertitletext{\vspace{-16bp}}
\usepackage[utf8]{inputenc}
\usepackage[T1]{fontenc}
\usepackage[protrusion=false]{microtype}
\usepackage[hidelinks]{hyperref}
\usepackage{amsmath,amssymb,amsfonts}
\usepackage{graphicx}
\usepackage{booktabs}
\usepackage{multirow}
\usepackage{xcolor}
\usepackage{tikz}
\usetikzlibrary{arrows.meta,positioning,fit,calc}
\usepackage[caption=false,font=footnotesize]{subfig}
\usepackage{balance}
\usepackage{url}
\usepackage{soul}
\usepackage{makecell}
\usepackage[normalem]{ulem}  %
\usepackage{xcolor}          %
\makeatletter
\@ifundefined{iflatexml}{\newif\iflatexml}{} 
\makeatother

\newtheorem{lemma}{Lemma}
\input{results_macros}

\usepackage{ulem}
\usepackage{xcolor}

\begin{document}

\title{\fontsize{16}{20}\selectfont\bfseries CALM: Decentralized Control of Cooperative Aerial Manipulation with Multi-Agent Reinforcement Learning}
\author{
\IEEEauthorblockN{
Jiaxi Huo,
Petre Ricioppo,
Fausto Francesco Lizzio$^{*}$
}
\thanks{The authors are with the Department of Mechanical and Aerospace Engineering, Politecnico di Torino, Torino, Italy (e-mail: \{jiaxi.huo, petre.ricioppo, fausto.lizzio\}@polito.it). }
\thanks{$^{*}$Corresponding author.}

}
\maketitle

\begin{abstract}
Cooperative aerial manipulation is difficult under wind, payload motion, and noisy goal broadcasts because multiple Unmanned Aerial Vehicles (UAVs) are coupled through a shared payload.
We present CALM (Cooperative Aerial Learned Manipulation), a decentralized controller for aerial manipulators carrying a basket while tracking a moving ground robot.
CALM combines a parameter-shared Multi-Agent Reinforcement Learning (MARL) policy, trained using Multi-Agent Proximal Policy Optimization (MAPPO), with goal estimation, per-UAV Extended State Observers (ESOs), basket-mass estimation, and anticipatory formation-yaw feedforward.
The anticipatory formation-yaw feedforward uses estimated goal motion to rotate the formation before position error develops, reducing basket-yaw lag without adding net commanded force.
A disturbance-conditioned critic supports centralized training while execution remains decentralized.
Across held-out trajectories, CALM achieves lower mean position and orientation errors than all evaluated baselines, including the matched arm baseline (\curveCalmMean{} versus \curvePriorMean\,cm and \curveCalmOri{} versus \curvePriorOri$^\circ$), the cable-suspended realization, and two model-based controllers.
Under unseen full-information-path degradation, CALM remains statistically comparable or superior to the arm baseline.
Videos are available at 
\href{https://petreric.github.io/CALM/}{https://petreric.github.io/CALM/}.

\end{abstract}

\section{Introduction}
\label{sec:introduction}

Transporting a shared payload with multiple Unmanned Aerial Vehicles (UAVs) is challenging because payload dynamics and interaction forces couple UAV motions.
These issues become even more prominent in outdoor agricultural environments, where the wind acts as a persistent disturbance, and variations in the payload mass distribution can cause the Center of Mass (CoM) to shift over time.

\begin{figure}[t]
\centering
\includegraphics[width=\columnwidth]
{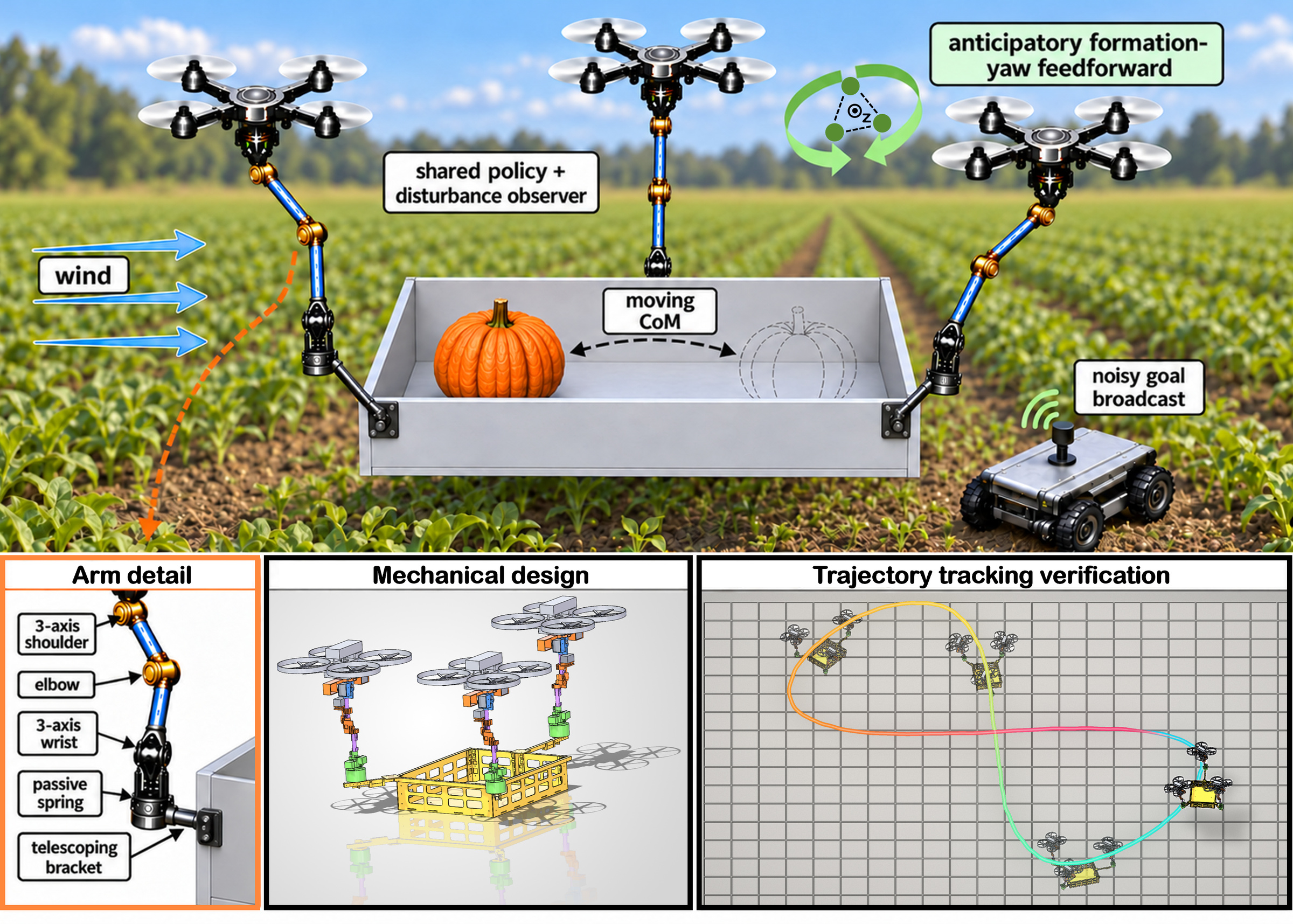}
\caption{
Illustration of three aerial manipulators carrying a shared basket while tracking a ground robot from noisy goal broadcasts under wind and a moving CoM. Lower insets show the developing hardware design and verification.
}
\label{fig:concept}
\end{figure}

Model-based cooperative transport controllers directly allocate payload wrenches and enforce coupling constraints, often with stability guarantees \cite{michael2011,lee2018}. However, they remain sensitive to model mismatch, compliance, and payload variation.

Learning-based control complements them by training across plant and disturbance distributions, as in recent work combining Reinforcement Learning (RL) with model-reference adaptation \cite{kannan2025}, robust model predictive control \cite{deng2025}, or control-barrier-constrained Multi-Agent RL (MARL) \cite{zhou2025}.
Decentralized MARL for cooperative transport has been explored solely on cable-suspended payloads \cite{zeng2025,lorentz2025}, whose unilateral interface leaves payload orientation weakly actuated.

Thus, in this work, we consider aerial manipulators, defined as UAVs equipped with an articulated arm that is rigidly coupled to the shared payload. This configuration provides direct control authority over the payload pose, but also introduces additional sources of complexity, including structural compliance, actuator saturation, and redundant interaction-force distribution.

We consider an air-ground setting in which three aerial manipulators carry a collection basket, a payload whose mass varies with its contents, across terrain with limited ground accessibility, while a ground robot performs the field task.
The aerial team tracks the robot from noisy goal broadcasts rather than a preplanned trajectory, exploiting its differential-drive motion for local coordinated-turn prediction.
Three aerial manipulators provide the minimum non-collinear formation for planar-force and yaw-moment allocation, while allowing asymmetric load sharing under a moving basket CoM. We therefore present CALM (Cooperative Aerial Learned Manipulation), a MARL-based decentralized controller.

Our contributions are as follows:
\begin{itemize}
    \item We formulate a decentralized cooperative aerial-manipulation task in which aerial manipulators carry a shared basket and track a ground robot from noisy goal broadcasts alone.
    \item We provide CALM with goal estimation, disturbance estimation, and an anticipatory formation-yaw feedforward.
    \item We validate CALM against arm-based and cable-suspended baselines, showing that it achieves the lowest position and orientation errors and retains its margin under degraded information paths.
\end{itemize}

This paper is organized as follows. In Sec.~\ref{sec:related}, a review of related works is provided. Sec.~\ref{sec:problem} outlines the multibody model and control objective, and Sec.~\ref{sec:method} presents CALM's control and learning architecture. In Sec.~\ref{sec:results}, CALM is evaluated against learned baselines and model-based controllers, while Sec.~\ref{sec:conclusion} provides some concluding remarks.

\section{Related Work}
\label{sec:related}

Classical work formulates cooperative aerial transport as grasp-map wrench allocation \cite{michael2011} or geometric control on the Special Euclidean group \(SE(3)\) \cite{lee2018}, later extended to model uncertainty \cite{sanalitro2020} and whole-body optimization \cite{sun2025}.
RotorTM \cite{li2024rotortm} cooperatively controls rigidly attached payloads.

To reduce model dependence, data-driven aerial control learns residual dynamics for model-based prediction \cite{belkhale2021,nieto2024safe} or maps observations directly to commands within adaptive or robust control \cite{liu2022ddpg}.
Examples include Deep Adaptive Trajectory Tracking with reference feedforward and disturbance estimation \cite{huang2023datt}, learned whole-body control for underactuated aerial manipulators \cite{deshmukh2026,singh2026}, and Gaussian-process compensation for leader-follower cooperative transport \cite{nieto2026leader}.

Decentralized MARL treats all UAVs as peers.
Centralized Training with Decentralized Execution (CTDE) trains each agent's policy, called the actor, with a critic that sees global information \cite{yu2022}.
Privileged critics that see training-only information \cite{baisero2022}, distributional critics \cite{dabney2018}, and domain randomization \cite{peng2018} have been studied individually.
The work of \cite{zeng2025} shows communication-free full-pose control of a cable-suspended payload using shared local policies.
CrazyMARL \cite{lorentz2025} extends it to direct motor commands and slack-taut cable transitions.
However, both \cite{zeng2025,lorentz2025} retain passive cable interfaces that transmit tension only, whereas rigid arms transmit bidirectional forces and moments \cite{kim2018,ollero2022}. Moreover, both provide the payload goal pose directly to the UAVs.

CALM instead combines explicit goal, mass, and disturbance estimates with a shared decentralized policy coordinating UAVs and arms without centralized online optimization.

\section{System Model and Problem Formulation}
\label{sec:problem}

\subsection{Coupled multibody plant}
Fig.~\ref{fig:concept} shows the multibody plant in which each UAV carries a two-link arm whose seven actuated joints comprise a three-axis shoulder, an elbow, and a three-axis wrist. A passive spring and a telescoping bracket join the wrist center, called the arm tip, to the basket.
The lag-limited compliant chain transmits torque and actuates the basket orientation indirectly.

Let \(N=3\) and let $i\in\{1,\ldots,N\}$ index the aerial manipulators.
The generalized coordinate \(q=[q_v,q_a,q_B]\) collects the three UAV poses, the arm coordinates, and the basket pose, respectively.
The coupled dynamics are
\begin{equation}
M(q)\ddot q+C(q,\dot q)\dot q+g_q(q)
=\mathcal B(q)u+\phi_c(q,\dot q)+\phi_d(q,\dot q,t).
\label{eq:multibody}
\end{equation}
Here, \(M(q)\), \(C(q,\dot q)\), and \(g_q(q)\) are the nominal inertial, Coriolis and centrifugal, and gravity terms, while \(\mathcal B(q)\) maps the UAVs and arms inputs \(u\).
The term \(\phi_c(q,\dot q)\) is the finite-stiffness arm-basket coupling, including the passive spring and bracket forces, and \(\phi_d(q,\dot q,t)\) contains wind and moving-CoM disturbances.
The basket rows of \eqref{eq:multibody}
\begin{align}
m_B\ddot p_B &= \textstyle\sum_i f_i + f_w + m_B\mathbf g, \label{eq:trans}\\
J_B\dot\omega_B+\omega_B\times J_B\omega_B
&=\textstyle\sum_i(r_i\times R_B^\top f_i+\tau_i)\notag\\
&\quad+\tau_w+\tau_{\rm CoM}.
\label{eq:rot}
\end{align}
are written in world-frame for translation and in basket-frame for rotation.
Here, \(p_B,\omega_B,m_B,J_B\) are the position, angular velocity, mass, and inertia of the basket with its contents, and \(\mathbf g\) is the gravity vector. Matrix \(R_B\) in the Special Orthogonal group \(SO(3)\) maps basket vectors to the world frame.
Arm \(i\) applies the wrench \((f_i,\tau_i)\) at its attachment point with basket-frame offset \(r_i\) from the basket CoM, and \((f_w,\tau_w)\) and \(\tau_{\rm CoM}\) are the wind wrench and moving-CoM torque modeled below.

\paragraph{Moving center of mass}
The CoM includes the basket contents. Its basket-frame offset $c_t\in\mathbb R^3$ at time $t$ follows an Ornstein-Uhlenbeck (OU) \cite{uhlenbeck1930theory} process at constant mass,
\begin{equation}
\mathrm dc_t=-\theta_c\,c_t\,\mathrm dt+\sigma_c\,\mathrm dW_t,
\label{eq:ou}
\end{equation}
where \(\theta_c>0\), \(\sigma_c\), and \(W_t\) are the mean-reversion rate, the diffusion scale, and the three-dimensional Wiener process.
The resulting stationary standard deviation \(\sigma_c/\sqrt{2\theta_c}\), called the CoM deviation, sets the amplitude of the CoM offset. However, the value \(c_t\) is bounded component-wise by vector \(\bar c=(\bar c_x,\bar c_y,\bar c_z)\), thereby bounding the induced gravity torque in the basket frame,
\begin{equation}
\tau_{\rm CoM}=c_t\times m_B R_B^\top\mathbf g.
\label{eq:comtorque}
\end{equation}

\paragraph{Wind disturbance}
The wind is modeled as a constant force plus gusts following an OU process similar to \eqref{eq:ou}.
On the basket, the disturbance enters \eqref{eq:trans}--\eqref{eq:rot} as
\((f_w,\tau_w)\), whereas on the UAVs it is included in
\(\phi_d(q,\dot q,t)\) in \eqref{eq:multibody}.

\subsection{Control objective}
\label{sec:sensing}

The ground robot generates online a goal \(g_t=(p_t^g,\psi_t^g)\) for the position and the yaw of the basket.
The desired altitude is fixed \textit{a priori}, while the desired pitch and roll angles are zero. However, wind, CoM torque, and unequal arm loads may tilt the basket, so the full orientation is tracked.
Each UAV receives \(g_t\) only through a broadcast with additive Gaussian position and yaw noise of standard deviations \(\sigma_p\) and \(\sigma_\psi\), and computes its commands on board. Then, the control objective is the decentralized tracking of the basket pose under the considered disturbances.
Each arm tip tracks its attachment point on a basket at \(g_t\), and each UAV chooses its own position.

\section{CALM Control and Learning Architecture}
\label{sec:method}

\begin{figure*}[t]
\centering
\input{fig_architecture.tex}
\caption{CALM architecture with per-UAV execution (solid), team-bus signals (blue dashed), training-only parts (amber), and unmeasured disturbances (dash-dot). Orange marks evaluation-only degradation of the full information path (Sec.~\ref{sec:info}). Symbols are defined in Secs.~\ref{sec:problem}--\ref{sec:method}.}
\label{fig:architecture}
\end{figure*}
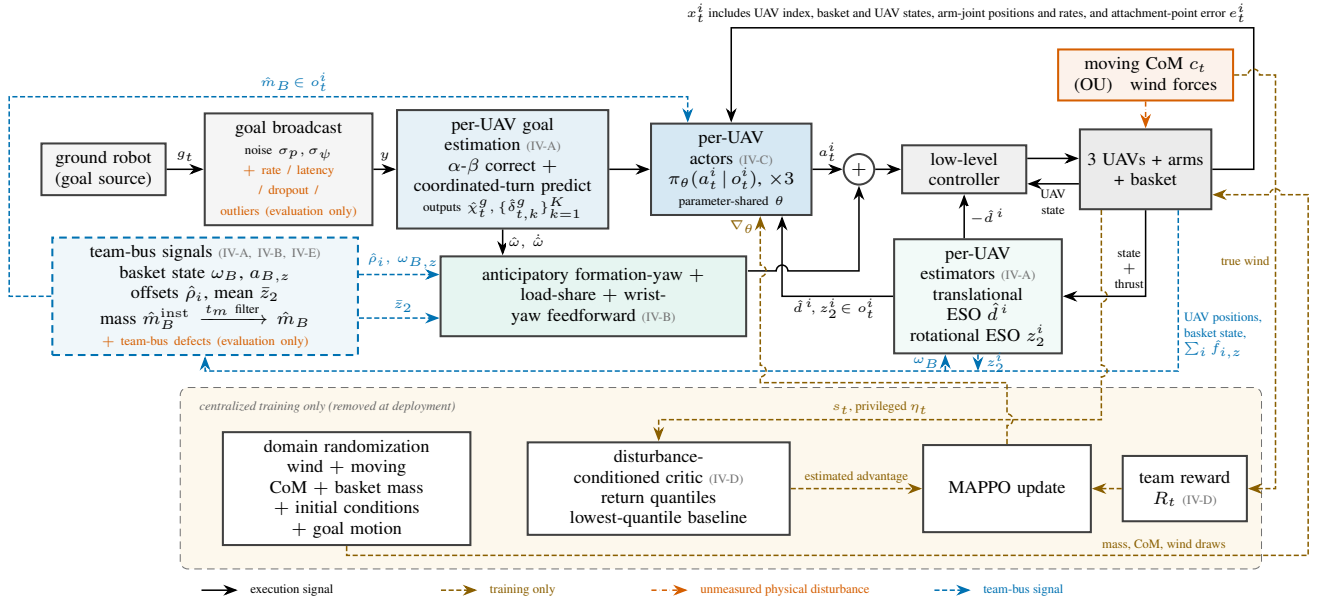

In this Section, an overview of CALM architecture is provided.
Fig.~\ref{fig:architecture} illustrates CALM's deterministic estimation and control modules (Secs.~\ref{sec:estimation}--\ref{sec:ff}), separately from its decentralized RL policy and centralized training (Secs.~\ref{sec:policy}--\ref{sec:critic}).
During training, the critic uses the centralized state, the privileged disturbances, and the shared reward to guide policy updates. During execution, the actor receives the estimated system information as part of its local observation, while the deterministic modules provide additional compensation to the commanded control input.

\subsection{Online estimation}
\label{sec:estimation}

Each UAV runs the following estimators for the goal $g_t$, the translational and rotational disturbances, and the basket mass.

\paragraph{Goal estimation}
An \(\alpha\)-\(\beta\) filter estimates the basket goal pose \(\hat g=(\hat p,\hat\psi)\) and rates \((\hat v,\hat\omega)\) from the noisy broadcast generated by the ground robot.
Each packet \(y=(y^p,y^\psi)\) is a noisy measurement of \(g_t=(p_t^g,\psi_t^g)\).
The filter propagates the estimates at constant velocity between packets. In particular, for each delivered packet, it predicts \(\hat p^-=\hat p+\hat v\Delta t_p\) over one policy period \(\Delta t_p\) and updates
\begin{equation}
\begin{alignedat}{2}
\hat p&\leftarrow\hat p^-+\alpha(y^p-\hat p^-),\enspace &
\hat v&\leftarrow\hat v+\tfrac{\beta}{\Delta t_p}(y^p-\hat p^-),\\
\hat\psi&\leftarrow\hat\psi+\alpha_\psi\mathrm{wrap}(y^\psi-\hat\psi),\enspace &
\hat\omega&\leftarrow\hat\omega+\alpha_\psi\big(\tfrac{\Delta\hat\psi}{\Delta t_p}-\hat\omega\big),
\end{alignedat}
\label{eq:goalfilter}
\end{equation}
where \(\mathrm{wrap}(\cdot)\) maps an angle to \([-\pi,\pi)\) and \(\Delta\hat\psi\) is the change of \(\hat\psi\) over the last period.
The basket position and orientation errors relative to \(\hat g\), together with \((\hat v,\hat\omega)\), form the goal-relative error vector \(\hat\chi_t^g\).

The current goal estimates are employed in a differential-drive (coordinated-turn) kinematic model to predict the goal pose at \(K\) future instants \(t+h_k\), \(k=1,\dots,K\).
Denoting as $\hat v_s$ and $\hat v_z$ the planar norm and vertical component of $\hat v$, the predicted goal pose is
\begin{align}
& \hat\psi_k=\hat\psi+\hat\omega h_k, \notag \\
& \hat p(t{+}h_k)=\hat p+
\begin{bmatrix}
(\hat v_s/\hat\omega)(\sin\hat\psi_k-\sin\hat\psi)\\
(\hat v_s/\hat\omega)(\cos\hat\psi-\cos\hat\psi_k)\\
\hat v_zh_k
\end{bmatrix}.
\label{eq:ctrv}
\end{align}
with the constant-velocity limit used as \(\hat\omega\to 0\).
At each look-ahead time \(h_k\), the predicted position offset and rotation vector relative to \(\hat g\) form the preview feature \(\hat\delta_{t,k}^g\in\mathbb R^6\).

\paragraph{On-board disturbance and basket-mass estimation}

Each UAV uses two Extended State Observers (ESOs) to estimate the unmodeled rotational dynamics of the basket and translational dynamics of the UAV \cite{cao2024}. The rotational ESO employs the measured angular velocity $\omega_B$ and the modeled angular acceleration \(u_m\) to estimate the basket angular velocity and the residual angular acceleration, which we represent as the observer states \(z_1,z_2\in\mathbb R^3\) with the UAV index omitted for clarity. The discrete-time dynamics evolve as
\begin{equation}
\begin{aligned}
z_1&\leftarrow z_1+\big(z_2+u_m+2\omega_o(\omega_B-z_1)\big)\Delta t_c,\\
z_2&\leftarrow z_2+\omega_o^2(\omega_B-z_1)\Delta t_c.
\end{aligned}
\label{eq:eso}
\end{equation}
Here, \(\Delta t_c\) is the low-level control period, \(\omega_o\) is the observer bandwidth, and \(u_m=J_B^{-1}(r_i\times R_B^\top f_i+\tau_i-\omega_B\times J_B\omega_B)\) keeps only the UAV's own arm wrench in \eqref{eq:rot}. The disturbance estimate is bounded before being used for compensation.
For a lumped angular-acceleration disturbance \(d_\omega\) with bounded rate \(\dot d_\omega\), the ideal continuous-time observer satisfies
\begin{equation}
\lVert z_2-d_\omega\rVert_\infty
\le \frac{2\lVert\dot d_\omega\rVert_\infty}{\omega_o},
\qquad
\lVert z_1-\omega_B\rVert_\infty
\le \frac{\lVert\dot d_\omega\rVert_\infty}{\omega_o^2},
\end{equation}
up to a decaying transient. These bounds determine the choice of \(\omega_o\).
The translational ESO follows the same observer structure. It uses the UAV velocity as the measured state and \(f_i^{\rm th}/m_v+\mathbf g\) as the modeled acceleration, where \(m_v\) is the UAV mass and \(f_i^{\rm th}\) is its world-frame thrust. Its second observer state provides the residual acceleration estimate \(\hat d^{\,i}\), which is bounded before being used to correct the low-level acceleration command and is also provided to the local policy.

The basket mass is estimated online because the objects it carries vary.
Each UAV first estimates the vertical force exerted by its arm from the UAV--arm force balance.
The estimates from all UAVs are then combined to recover the total vertical force acting on the basket. Assuming negligible vertical wind, this force together with the measured basket vertical acceleration provides
an instantaneous estimate of the basket mass.

Let \(e_3=[0,0,1]^\top\) denote the world-frame vertical unit vector. Let \(p_i\) and \(p_\ell\) be the world-frame CoM positions of UAV \(i\) and arm body \(\ell\), respectively, and define \(a_{(\cdot),z}=e_3^\top\ddot p_{(\cdot)}\) for \(p_i,p_B,p_\ell\). We also define \(T_{i,z}=e_3^\top f_i^{\rm th}\) and \(g_0=-e_3^\top\mathbf g>0\).
The vertical force \(e_3^\top f_i\) that arm \(i\) exerts on the basket, with \(f_i\) from \eqref{eq:trans}, is estimated as
\begin{equation}
\hat f_{i,z}
=
T_{i,z}
-m_s(a_{i,z}+g_0)
-\sum_{\ell\in\mathcal L_i}m_\ell(a_{\ell,z}+g_0),
\end{equation}
where \(m_s\) is the known mass of the UAV and its co-located arm actuators, and \(\mathcal L_i\) contains the arm bodies \(\ell\) of mass \(m_\ell\). Summing the force estimates across the UAVs and applying the vertical component of \eqref{eq:trans} gives the instantaneous basket-mass estimate
\begin{equation}
\hat m_B^{\rm inst}
=
\frac{\sum_i\hat f_{i,z}}{a_{B,z}+g_0}.
\end{equation}
Only the scalar \(\sum_i\hat f_{i,z}\) is exchanged between the agents, so no individual force estimate is shared. The resulting estimate is guarded against small denominators, clipped to \([m_{\min},m_{\max}]\), and smoothed into \(\hat m_B\) by a filter of time constant \(t_m\) that updates only near hover.

\subsection{Deterministic feedforward compensation}
\label{sec:ff}

The wrists have limited yaw range, so the basket yaw is primarily controlled through rotation of the entire formation. Without anticipatory feedforward, the formation begins to rotate only in response to the developing tracking error, resulting in yaw-tracking lag during turns. The filter in Section \ref{sec:estimation} already provides an estimate of the turn rate \(\hat{\omega}\), which CALM uses to anticipate the required formation rotation.
The demanded moment is
\begin{equation}
M_z^{\rm ff}
=
I_{\rm eff}
\big[
\dot{\hat\omega}
+
K_\omega
(\hat\omega+t_{\rm lead}\dot{\hat\omega}-\omega_{B,z})
\big],
\label{eq:yawmoment}
\end{equation}
where \(K_\omega\) is the yaw-rate feedback gain, \(t_{\rm lead}\) determines the amount of rate anticipation, and
\begin{equation}
I_{\rm eff}
=
m_v\sum_{j=1}^{N}\lVert\hat\rho_j\rVert^2+J_{B,z}
\end{equation}
accounts for both the formation inertia and the basket yaw inertia. Here, \(\hat\rho_j\) is the planar position of UAV \(j\) relative to the formation centroid, obtained from the UAV position estimates exchanged through the team bus (Sec.~\ref{sec:info}), \(J_{B,z}\) is the basket yaw moment of inertia, and \(\omega_{B,z}\) is the basket yaw rate about \(e_3\). The term \(\dot{\hat\omega}\) is obtained by filtering the finite difference of the estimated goal turn rate.

The desired moment is realized by assigning tangential accelerations to the UAVs according to
\begin{equation}
a_i^{\rm yaw}
=
\frac{M_z^{\rm ff}}
{m_v\sum_j\lVert\hat\rho_j\rVert^2}\,
e_3\times\hat\rho_i,
\qquad
\sum_i a_i^{\rm yaw}=0.
\label{eq:yawallocation}
\end{equation}
This allocation produces the desired rotational effect without introducing a net translational acceleration of the formation. Since \(\hat\rho_i\) comes from estimated UAV positions, Lemma~\ref{prop:alloc} quantifies the resulting yaw-moment error.

\begin{lemma}[Allocation under estimated geometry]
\label{prop:alloc}
Let \(\rho_i\) be the true horizontal UAV offsets from the formation centroid and \(\hat S=\sum_j\lVert\hat\rho_j\rVert^2\).
For any \(\rho_i\), \eqref{eq:yawallocation} yields zero net force, and its yaw moment about the centroid is \(M_z^{\rm alloc}=\kappa M_z^{\rm ff}\), with \(\kappa=\hat S^{-1}\sum_i\rho_i\cdot\hat\rho_i\).
If \(\max_i\lVert\rho_i-\hat\rho_i\rVert\le\varepsilon\), then \(\lvert\kappa-1\rvert\le\varepsilon\sum_i\lVert\hat\rho_i\rVert/\hat S\), which equals \(\varepsilon/\hat r\) for equal estimated radii \(\hat r\).
\end{lemma}

The proof is provided in Appendix~\ref{app:lemma-proof}. The allocated yaw moment thus keeps the sign of \(M_z^{\rm ff}\) if the bound is below one.

Two more feedforward paths use the team mean \(\bar z_2=N^{-1}\sum_i z_2^i\) of the rotational ESO states on the team bus.
Its roll and pitch components set zero-sum vertical-acceleration offsets forming a load-share couple against \(J_B\bar z_2\), and its yaw component adds a common offset to all wrist-yaw targets.
Both are filtered, ramped, clipped, and active in training.

\subsection{Decentralized RL policy}
\label{sec:decpomdp}
\label{sec:policy}

At each policy step, agent \(i\) builds its local observation from its own state and the online estimates of Sec.~\ref{sec:estimation} as
\begin{align}
\xi^i_t&=\big(x^i_t,\;\hat d^{\,i},\;z^i_2,\;\hat m_B,\;\hat\chi_t^g,\;
\{\hat\delta_{t,k}^g\}_{k=1}^{K}\big),\nonumber\\
o^i_t&=\big(\xi^i_t,\xi^i_{t-1},\xi^i_{t-2}\big),
\label{eq:obs}
\end{align}
where \(o_t^i\) stacks \(\xi_t^i\) for three consecutive time frames. The block \(x_t^i\) contains the UAV index, the basket and UAV states, seven arm-joint positions and rates, and the attachment-point error \(e_t^i=R_B^\top(p_t^{{\rm tip},i}-\hat p-R_z(\hat\psi)\,r_i)\). Here, \(p_t^{{\rm tip},i}\) is the arm-tip position and \(R_z(\hat\psi)\) is the yaw rotation of \(\hat g\).
The remaining blocks are the ESO outputs \(\hat d^{\,i}\) and \(z^i_2\), the basket-mass estimate \(\hat m_B\), all read at step \(t\), the goal-relative error vector \(\hat\chi_t^g\), and the \(K\) preview features \(\hat\delta_{t,k}^g\).
The action \(a^i_t\in\mathbb R^{12}\) has three acceleration references, roll- and pitch-rate references with UAV yaw held fixed, and seven arm-joint targets, tracked by a faster low-level controller.

The parameter-shared Multi-Layer Perceptron (MLP) actor
\(\pi_\theta(a_t^i\mid o_t^i)\) maps this observation to the local action,
with Multi-Agent Proximal Policy Optimization (MAPPO) optimizing the shared policy parameters \(\theta\) under CTDE
\cite{yu2022,schulman2017}.
Attachment geometry and UAV index distinguish the roles, and
\(\nabla_\theta\) in Fig.~\ref{fig:architecture} denotes the gradient of the actor objective.

\subsection{Centralized RL training}
\label{sec:critic}

\paragraph{Distributional critic}
The distributional MAPPO critic is disturbance-conditioned and used only during training. It receives the centralized state \(s_t\), which comprises the UAV, arm, basket, and goal states, together with a privileged vector \(\eta_t\) of the wind forces acting on the UAVs and the basket and the CoM torque, which the actors never observe.
It predicts \(N_q\) return quantiles \(Z_j(s_t,\eta_t)\), \(j=1,\ldots,N_q\), by quantile-Huber regression \cite{dabney2018}, with \(\lambda\)-return targets from a first Generalized Advantage Estimation (GAE) pass bootstrapped from their mean.
Once the curriculum (Sec.~\ref{sec:protocol}) ends, a second GAE pass for the actor's advantages ramps its baseline from this mean to the mean of the lowest \(\kappa_qN_q\) quantiles, \(\kappa_q\in(0,1]\).
\label{sec:reward}

\paragraph{Team reward}
All agents receive a multiplicative team reward adapted from~\cite{lorentz2025},
\begin{equation}
R_t=\left(w_p\log\frac{\bar e_p}{\tilde e_{p,t}}
+w_o\log\frac{\bar e_o}{\tilde e_{o,t}}\right)r^{\rm cm}_t+R^{\rm aux}_t,
\label{eq:reward}
\end{equation}
where \(w_p,w_o>0\) are tracking weights, and \(r_t^{\rm cm}\) and \(R_t^{\rm aux}\) are defined below. Then, let  \(p_{B,t},p_t^g\) and \(R_{B,t},R_{g,t}\) be the current and the desired world-frame basket positions and orientations, and let the trace of \(\Delta R_t=R_{g,t}^{\top}R_{B,t}\) represent their angular mismatch. Introduce tracking errors lower and upper bounds \(\varepsilon_p,\varepsilon_o>0\) and \(\bar e_p,\bar e_o\). Then, the error terms in \eqref{eq:reward} are
\[
\begin{gathered}
\tilde e_{p,t}=\operatorname{clamp}
(\lVert p_{B,t}-p_t^g\rVert,\varepsilon_p,\bar e_p),\\
e_{o,t}=\cos^{-1}\!\operatorname{clamp}
\!\left(\frac{\operatorname{tr}(\Delta R_t)-1}{2},-1,1\right),\\
\tilde e_{o,t}=\operatorname{clamp}
(e_{o,t},\varepsilon_o,\bar e_o).
\end{gathered}
\]
During aggressive maneuvers, the UAVs must tilt to generate the accelerations required for basket tracking. 
Directly penalizing tilt would therefore conflict with tracking and could also interfere with the different attitudes required for arm coordination. 
To discourage unnecessary tilt after tracking converges, we introduce the common-mode gate \(r_t^{\rm cm}\). 
The gate remains one during the initial tracking phase and decreases toward a positive floor when the team-mean inclination of the UAV thrust axes, called the common-mode tilt, exceeds a deadband. 
The deadband adapts to the goal acceleration and training-only wind disturbance, so the tilt required by the maneuver or external disturbance is not penalized.
By \eqref{eq:trans}--\eqref{eq:rot}, arm \(i\) contributes the basket linear and angular accelerations
\begin{equation}
  a_i^{\rm arm}=R_B^\top f_i/\hat m_B, \quad
  \dot\omega_i^{\rm arm}=J_B^{-1}(r_i\times R_B^\top f_i+\tau_i).
\end{equation}
All accelerations below are expressed in the basket frame, with subscript \(z\) denoting the basket-frame \(z\)-component.
For desired basket accelerations \(a_t^{\rm des}\) and \(\dot\omega_t^{\rm des}\) from a Proportional-Derivative (PD) law on tracking errors, the auxiliary reward \(R_t^{\rm aux}\) has three clipped terms proportional to
\begin{equation}
  \sum_i(a_i^{\rm arm})^\top a_t^{\rm des},\quad
  \sum_i(\dot\omega_i^{\rm arm})^\top\dot\omega_t^{\rm des},\quad
  \sum_i\dot\omega_{i,z}^{\rm arm}\dot\omega_{t,z}^{\rm des}.
\end{equation}
The angular terms exploit attachment moments unavailable to cables.
Other terms penalize arm/UAV proximity and basket downwash, with collisions above \(2\,\mathrm{N}\) terminating the episode.

\subsection{Information pattern}
\label{sec:info}
Deployed actors receive no peer observations or actions and filter the shared goal locally.
The team bus carries only UAV positions for \(\hat\rho_i\), the basket Inertial Measurement Unit (IMU) state including \(\omega_B\), \(\sum_i\hat f_{i,z}\) and the averaged ESO state \(\bar z_2\).
The goal broadcast and bus form the full information path, with rate, latency, dropout, and outlier defects unseen in training.

\section{Simulation Results}
\label{sec:results}
\label{sec:experiments}

\subsection{Training, evaluation protocol, and compared controllers}
\label{sec:protocol}
\label{sec:compared}

We extend the public Isaac Lab code~\cite{zeng2025,aerialmaniplab2025marl}, keeping its Agilicious-based UAV model~\cite{foehn2022agilicious} and replacing the cable-suspended plant with the aerial-manipulator system of Sec.~\ref{sec:problem}.
Policies are trained from scratch in parallel environments with a curriculum that ramps wind, redraws the steady wind every episode, and randomizes basket mass, moving CoM, initial conditions, and irregular differential-drive goal motion \cite{peng2018}.
A fixed tracking metric on the training distribution selects each learned method's checkpoint, and Table~\ref{tab:params} lists the parameters.

\begin{table}[t]
\caption{Parameter values. $\mathcal U$ is a uniform draw, resampled per episode.}
\label{tab:params}
\centering
\fontsize{8pt}{9pt}\selectfont
\renewcommand{\arraystretch}{0.93}
\begin{tabular}{@{}p{0.31\columnwidth} p{\dimexpr0.69\columnwidth-2\tabcolsep\relax}@{}}
\toprule
Policy / low-level rate & 100/300\,Hz; simulation step 3.33\,ms; episode 80\,s \\
Arm joints (actuated) & 3-axis shoulder ($\pm$90$^\circ$), elbow (0--180$^\circ$),
3-axis wrist ($\pm$90$^\circ$); two 0.125-m links, 50\,g each \\
Joint effort limits & 20\,N\,m shoulder/elbow, 5\,N\,m wrist, 20\,N telescoping \\
Spring / bracket & setpoint PD; spring 200\,N\,m/rad, damping 20;
bracket 250\,N/m, damping 25, travel $\pm$0.10\,m \\
Airframe / total mass & 0.602\,kg per airframe, $\approx$6.57\,kg total \\
\midrule
Wind mean force & $\le$1.15\,N per aerial manipulator and basket
(drag-equivalent wind speed $\approx$\windSpeedTrain\,m/s, curriculum-ramped) \\
Wind turbulence & OU noise, 0.575-N standard deviation, scaled in sweeps; rate
12\,s$^{-1}$; weight 0.9 on a part shared by all bodies \\
Moving CoM, Eq.~\eqref{eq:ou} & $\theta_c{=}1.5$\,s$^{-1}$, CoM deviation
0.10\,m (varied), CoM clipped to $\bar c{=}(10,10,8)$\,cm componentwise \\
\midrule
Goal / broadcast & altitude $\mathcal U[1.4,1.6]$\,m, fixed per episode;
Gaussian noise $\sigma_p{=}0.05$\,m, $\sigma_\psi{=}0.05$\,rad, one packet per policy step \\
Degraded goal broadcast & eval. period $\mathcal U\{1..10\}$ steps;
latency $\mathcal U\{0..8\}$ steps; dropout $3{\times}10^{-3}$/step;
outliers $p{=}10^{-2}$, $10\sigma_p$ \\
Degraded team bus & as above, last packet held; 0.01-m position quantization \\
Goal estimation & $\alpha{=}\alpha_\psi{=}0.2$, $\beta{=}0.02$;
$K{=}8$, $h_k{=}0.5k$\,s \\
\midrule
{Formation-yaw feedforward} &
yaw-rate lead time $t_{\rm lead}{=}0.15$\,s, rate gain $K_\omega{=}3$\,s$^{-1}$ \\
Observers & translational/rotational ESO 20/15\,rad/s;
$t_m{=}1$\,s, $[m_{\min},m_{\max}]{=}[0.5,2.0]$\,kg \\
\midrule
Reward shaping & $w_p{=}2$, $w_o{=}1$;
$\varepsilon_p{=}10^{-3}$\,m, $\varepsilon_o{=}10^{-3}$\,rad;
$\bar e_p{=}3$\,m, $\bar e_o{=}3$\,rad; $r_t^{\rm cm}\in[0.1,1]$ \\
Learning & MAPPO, discount $\gamma{=}0.99$, GAE $\lambda{=}0.95$, clip ratio 0.1; 2048 environments; MLP 1024-512-256-128 \\
Actor observation & 123 features/frame; three-frame stack, \(o_t^i\in\mathbb R^{369}\) \\
Critic & $N_q{=}16$, $\kappa_q{=}0.25$; $\dim s_t{=}187$, $\dim\eta_t{=}15$ \\
\bottomrule
\end{tabular}
\vspace{-4pt}
\end{table}

\begin{figure*}[tp]
\centering
\includegraphics[width=\textwidth]{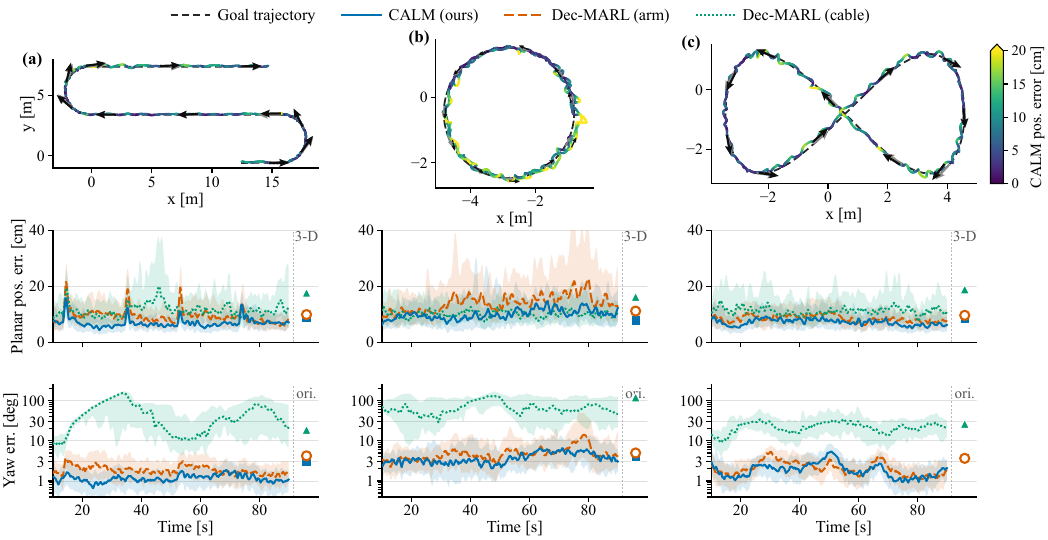}
\caption{Tracking under nominal disturbance. The top panel shows a representative CALM path colored by position error, with the dashed goal trajectory and commanded (gray) and realized (black) basket yaw. The lower panels show median planar position and yaw errors over \traceEnvs{} trace environments recorded separately from the statistical pool, with interquartile shading. Right markers give scoring-window medians of 3-D position and full orientation error.}
\label{fig:tracking}
\end{figure*}

\begin{table*}[tp]
\centering
\caption{Held-out position [cm] / orientation [deg] errors, given as mean $\pm$ SD across \trialsPerCell{} trials of the per-trial environment-time medians.}
\label{tab:matrix}
\setlength{\tabcolsep}{2.4pt}
\resizebox{\textwidth}{!}{%
\begin{tabular}{lcccccc}
\toprule
Method & Boustrophedon & Circle & Figure-8 & Perimeter & Contour &
Chevron\\
\midrule
\input{table_main_body}
\end{tabular}}
\end{table*}

\begin{figure*}[!tp]
\centering
\includegraphics[width=\textwidth]{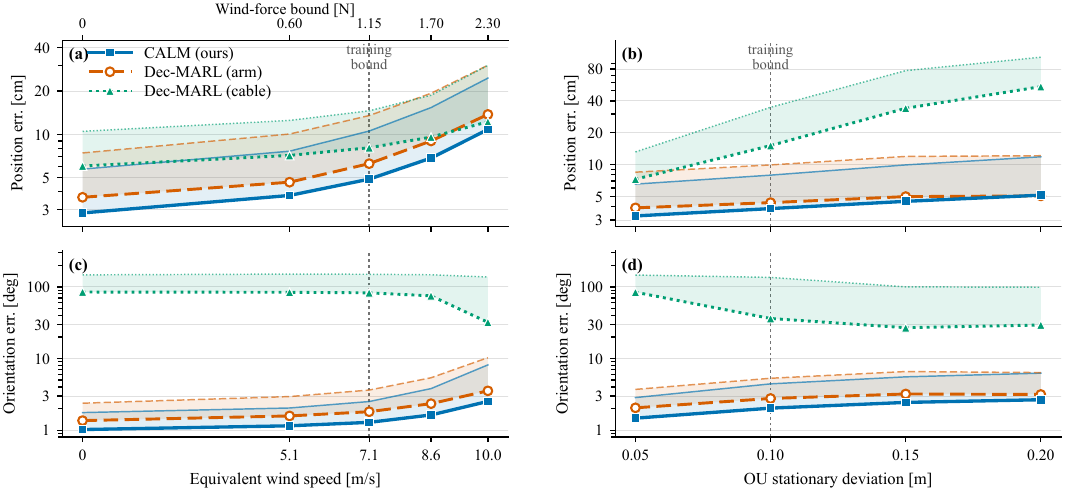}
\caption{Position (a,b) and orientation (c,d) sweeps under wind (left) and CoM deviation (right). Equivalent wind speed matches total mean applied force to system drag, with per-body force on the top axis. Curves show across-trial medians of the per-trial p50 and p90, with \(n=\trialsPerCell\) trials per method.}
\label{fig:sweeps}
\end{figure*}

Table~\ref{tab:matrix} evaluates the six held-out trajectory families sketched in Fig.~\ref{fig:fulldeg} under nominal disturbance, which is the training-bound wind with 0.10-m CoM deviation.
Each condition has \trialsPerCell{} trials with distinct seeds. A trial pools \envsPerTrial{} environments over a 15-s scoring window after a 4-s settle, and its p50 and p90 are the 50th and 90th error percentiles. Table~\ref{tab:matrix} gives the mean and Standard Deviation (SD) of the trial p50.
Figure~\ref{fig:sweeps} independently sweeps wind to twice its training bound and CoM deviation to 0.20\,m over 60-s boustrophedon trials. 
Since the equivalent wind speed matches the total applied force to the system drag, the basket receives a smaller force share than under physical wind.
Arm-platform trials are paired under matched resets and disturbances, whereas the cable code of \cite{zeng2025,aerialmaniplab2025marl} yields distinct episodes per seed.
The 95\% Confidence Intervals (CIs) use 20\,000 percentile-bootstrap resamples of paired differences.

\begin{figure}[tp]
\centering
\subfloat[]{\includegraphics[height=0.286\columnwidth]{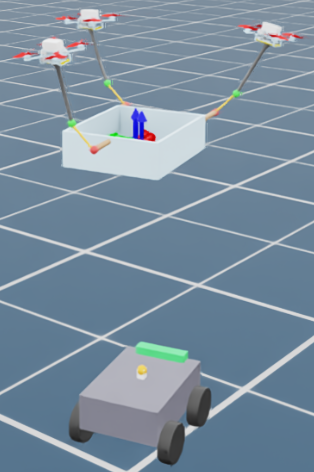}}\hfill
\subfloat[]{\includegraphics[height=0.286\columnwidth]{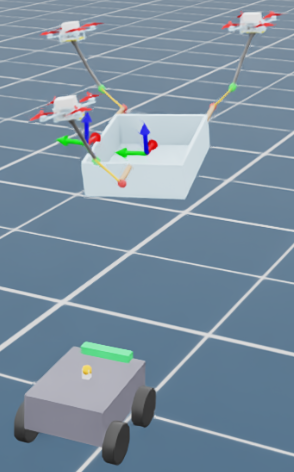}}\hfill
\subfloat[]{\includegraphics[height=0.286\columnwidth]{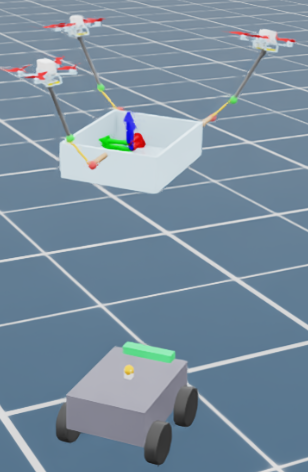}}\hfill
\subfloat[]{\includegraphics[height=0.286\columnwidth]{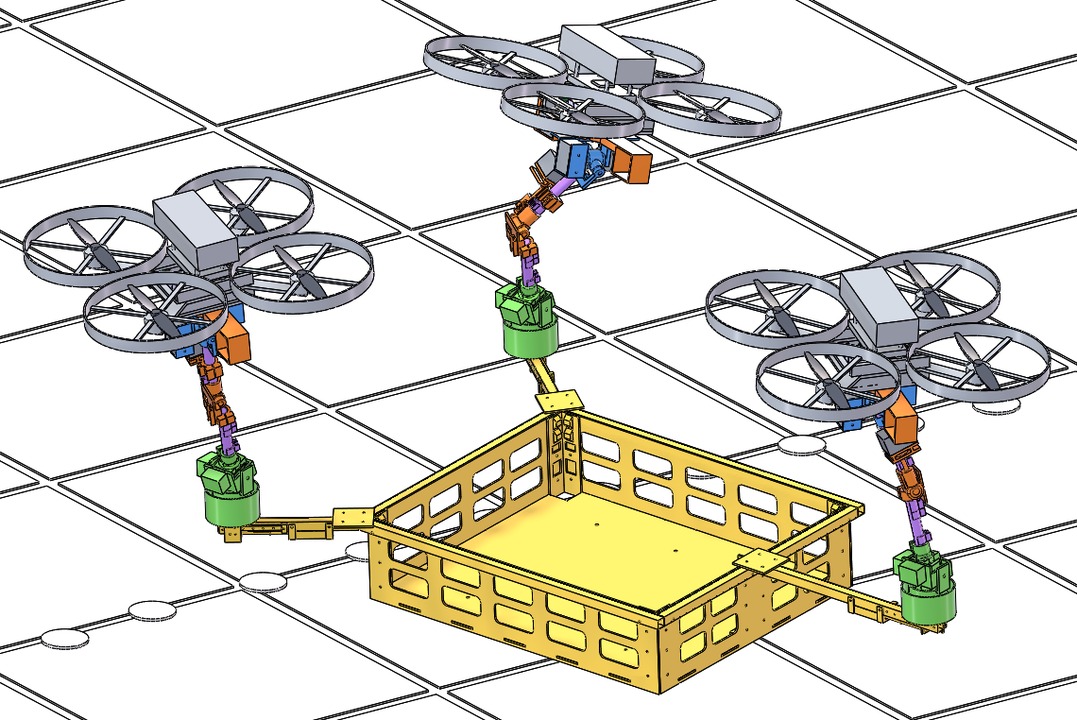}}
\caption{Isaac Sim snapshots at (a) the training wind bound (1.15\,N, \windSpeedTrain\,m/s equivalent), (b) twice that bound (2.30\,N, \windSpeedHigh\,m/s equivalent), and (c) \mbox{0.20-m} CoM deviation. Markers give goal and basket pose. In (d), SolidWorks verifies that the hardware design follows a CALM flight under wind, gusts, and the moving CoM within its joint travel and without part interference.}
\label{fig:simframes}
\end{figure}

\emph{Dec-MARL (arm)}, a decentralized MARL baseline, adapts \cite{zeng2025,aerialmaniplab2025marl} to CALM's platform, sharing the curriculum, plant, actor observations, observers, goal estimation, critic network, low-level controller, and translational ESO correction. Its original reward is kept, with the linear term of \(R_t^{\rm aux}\) added to induce arm motion. It omits the angular terms, formation-yaw feedforward, load-share couple, wrist-yaw offset, and privileged vector \(\eta_t\).

\emph{Dec-MARL (cable)} uses the open-source cable code and the reward of \cite{zeng2025}, retrained under CALM's disturbance curriculum.
The model-based controllers are \emph{Grasp allocation}, a quasi-static wrench allocation with PD formation in the grasp-map form of \cite{michael2011}, and \emph{RotorTM} \cite{li2024rotortm}, a cooperative controller for rigidly attached payloads.
Both use ground-truth state.

\subsection{Differential-drive ground robot tracking}
\label{sec:trackingresults}

CALM attains the lowest position and orientation errors on every trajectory family in Table~\ref{tab:matrix}, and paired comparisons favor it over Dec-MARL (arm) in all \pairedNomPositive{} families.
The margin is smallest on the \pairedNomSmallestCurve{}, whose long straight phases suit both arm controllers, and largest on the circle, perimeter, and contour, where sustained turns and U-turns dominate (Fig.~\ref{fig:tracking}).

Altitude error increases the 3-D position metric of Dec-MARL (cable) by 6--9\,cm relative to the planar trace, and its orientation error on the circle reaches \(123.2^\circ\) versus \(2.4^\circ\)--\(4.4^\circ\) for CALM.
The model-based controllers remain near \curveGraspMean{} and \curveRotorMean\,cm in position, with orientation errors of \curveGraspOri$^\circ$ and \curveRotorOri$^\circ$, respectively. Their low across-family SDs of \curveGraspSD{} and \curveRotorSD\,cm indicate that they fail consistently.
The anticipatory feedforward alone reduces the basket orientation error in the training distribution by 10--16\% at the median and 8--18\% at p90, mainly during turns, with no consistent position effect.

\subsection{Wind and moving-CoM sweeps}

In Fig.~\ref{fig:sweeps}(a) and Fig.~\ref{fig:simframes}(a,b), both arm controllers degrade with wind, and CALM's margin over Dec-MARL (arm) widens. CALM remains lower at every level tested, with every 95\% CI excluding zero.
At an equivalent wind speed of \windSpeedHigh\,m/s, Dec-MARL (cable) is less accurate than CALM, with a median UAV tilt of \tiltCable$^\circ$ and uncontrolled basket orientation.

From 0.05 to 0.20\,m CoM deviation (Fig.~\ref{fig:simframes}(c)), CALM's position error rises moderately, while Dec-MARL (cable) degrades steeply in position and remains at least \cableOriRatioMin$\times$ worse in orientation.
At 0.20\,m, Dec-MARL (arm) is marginally better than CALM, with a paired difference of \mbox{\comHighPairedDelta\,cm} (CI \mbox{[\comHighPairedLo,~\comHighPairedHi]}).
Clipping the CoM offset at \(\bar c\) places both arm controllers in the same torque-saturation regime. As a result, CALM's largest margins arise on turn-rich trajectories and in strong wind, and both controllers hold basket orientation.

From zero wind to twice the training bound (\windSpeedHigh\,m/s equivalent), CALM's median UAV tilt rises from \tiltCalmMin{} to \tiltCalmMax$^\circ$ and stays about \tiltPriorDelta$^\circ$ above that of Dec-MARL (arm), while its tilt p90 reaches \tiltCalmNinetyMax$^\circ$.
Under nominal disturbance, its position p90 is \pNinetyCalmMin--\pNinetyCalmMax\,cm versus \pNinetyPriorMin--\pNinetyPriorMax\,cm for Dec-MARL (arm).
In Fig.~\ref{fig:simframes}(d), SolidWorks verifies the hardware design.

\subsection{Arm joint-group ablation}
\label{sec:ablation}

All actuated joints move under nominal disturbance with mean torques close to their limits, and at 0.20-m CoM deviation the telescoping brackets mostly sit at their travel limits.
Table~\ref{tab:ablation} reports the effect of unpowering each joint group by zeroing its position-control stiffness over \(n=64\) paired trials.
Unpowering the full arm raises position and orientation errors the most, mainly through wrist roll and pitch, and unpowering only the shoulder and elbow still degrades both.
Wrist yaw matters less despite frequent saturation, as formation rotation provides the basket yaw (Sec.~\ref{sec:ff}).
The elbow is active, but unpowering it alone has no significant effect, as the attachment height follows the desired altitude fixed \textit{a priori} (Sec.~\ref{sec:sensing}).

\begin{figure}[t]
\centering
\includegraphics[width=\columnwidth]{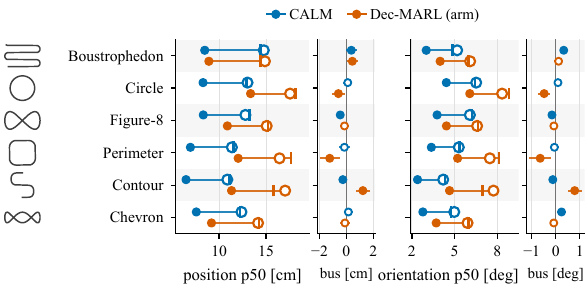}
\caption{Position and orientation with undegraded (filled), goal-degraded (tick), and goal-and-bus-degraded (open) information paths in \trialsPerCell{} paired trials. Narrow panels show bus-only changes with 95\% CIs. Filled ones exclude zero.}
\label{fig:fulldeg}
\end{figure}

\begin{table}[t]
\centering
\caption{Median-error ratio of each joint-group ablation to the matched unmodified controller, with bootstrap 95\% CIs.}
\label{tab:ablation}
\setlength{\tabcolsep}{4pt}
\footnotesize
\begin{tabular}{lcc}
\toprule
Unpowered group & Position & Orientation \\
\midrule
\input{table_ablation_body}
\end{tabular}
\end{table}

\subsection{Degradation of the full information path}
\label{sec:sensingresults}

In Fig.~\ref{fig:fulldeg}, the goal broadcast and the team bus are degraded as listed in Table~\ref{tab:params}, and none of these defects occurs in training. The UAV positions on the bus also carry a \busPosNoise-m localization error.
CALM leads on \fdCalmWins{} of \fdCurveCount{} trajectories in position and \fdCalmWinsOri{} of \fdCurveCount{} in orientation.
When only the team bus is degraded, the CIs include zero for both controllers, with a centroid error of \busCentroidErr\,cm and packet staleness of up to \busStaleWorst\,s.
Most degradation thus comes from the goal broadcast.

\section{Conclusion}
\label{sec:conclusion}

CALM, which combines a decentralized MARL with deterministic estimation and feedforward, outperforms all evaluated baselines.
CALM is decentralized in execution but not fully distributed, since its feedforward paths and mass estimate use team-bus quantities.
Future work will estimate them by dynamic average consensus over neighbor links and address hardware validation and formal closed-loop safety guarantees.

\appendices
\section{Proof of Lemma~\ref{prop:alloc}}
\label{app:lemma-proof}
\begin{IEEEproof}
By \eqref{eq:yawallocation}, the feedforward assigns to UAV \(i\) the horizontal force \(F_i=m_va_i^{\rm yaw}=(M_z^{\rm ff}/\hat S)\,e_3\times\hat\rho_i\), which uses the estimated offsets \(\hat\rho_i\) but acts at the true offsets \(\rho_i\).

\emph{Net force.} Each \(\hat\rho_i\) is relative to the mean of the estimated UAV positions, so \(\sum_i\hat\rho_i=0\) and \(\sum_iF_i=0\) for any true \(\rho_i\).

\emph{Yaw moment.} By the vector triple product, \(\rho_i\times(e_3\times\hat\rho_i)=(\rho_i\cdot\hat\rho_i)\,e_3-(\rho_i\cdot e_3)\,\hat\rho_i\).
Since \(\hat\rho_i\) is horizontal, the second term has no yaw component, so the total yaw moment about the centroid is
\begin{equation*}
M_z^{\rm alloc}=e_3\cdot\sum_i\rho_i\times F_i
=\frac{M_z^{\rm ff}}{\hat S}\sum_i\rho_i\cdot\hat\rho_i=\kappa M_z^{\rm ff},
\end{equation*}
and exact geometry \(\rho_i=\hat\rho_i\) gives \(\kappa=1\).

\emph{Gain bound.} Substituting \(\rho_i=\hat\rho_i+\tilde\rho_i\), where \(\tilde\rho_i\) is the offset error, into \(\kappa\) gives \(\kappa=\hat S^{-1}\sum_i(\hat\rho_i+\tilde\rho_i)\cdot\hat\rho_i=1+\hat S^{-1}\sum_i\tilde\rho_i\cdot\hat\rho_i\).
The Cauchy--Schwarz bound \(\lvert\tilde\rho_i\cdot\hat\rho_i\rvert\le\varepsilon\lVert\hat\rho_i\rVert\) then gives
\begin{equation*}
\lvert\kappa-1\rvert
\le\frac{1}{\hat S}\sum_i\lvert\tilde\rho_i\cdot\hat\rho_i\rvert
\le\frac{\varepsilon}{\hat S}\sum_i\lVert\hat\rho_i\rVert.
\end{equation*}
For equal estimated radii \(\hat r\), the bound is \(\varepsilon N\hat r/(N\hat r^2)=\varepsilon/\hat r\).
\end{IEEEproof}

\section*{Acknowledgement}
ChatGPT 6 Astra (OpenAI) was used to generate the upper-panel background and the lower-left arm detail in Fig.~\ref{fig:concept}.
The authors thank the Aerial Manipulation Lab for the framework in \cite{zeng2025,aerialmaniplab2025marl} and Prof. Elisa Capello, Head of the STREAM Robotics Lab at Politecnico di Torino, for her valuable guidance and support.

\bibliographystyle{IEEEtran}
\bibliography{references}
\end{document}

%% file: results_macros.tex
\newcommand{\curveCalmMean}{7.68}

\newcommand{\curvePriorMean}{10.95}

\newcommand{\curveCalmOri}{3.3}

\newcommand{\curvePriorOri}{4.7}

\newcommand{\curveGraspMean}{53}
\newcommand{\curveGraspSD}{1.9}
\newcommand{\curveRotorMean}{74}
\newcommand{\curveRotorSD}{1.5}
\newcommand{\curveGraspOri}{75}
\newcommand{\curveRotorOri}{92}

\newcommand{\windSpeedTrain}{7.1}
\newcommand{\windSpeedHigh}{10.0}

\newcommand{\cableOriRatioMin}{11}
\newcommand{\pNinetyCalmMin}{18}
\newcommand{\pNinetyCalmMax}{26}
\newcommand{\pNinetyPriorMin}{29}
\newcommand{\pNinetyPriorMax}{46}
\newcommand{\pairedNomSmallestCurve}{boustrophedon}

\newcommand{\pairedNomPositive}{6}

\newcommand{\comHighPairedDelta}{-0.08}
\newcommand{\comHighPairedLo}{-0.10}
\newcommand{\comHighPairedHi}{-0.06}
\newcommand{\tiltCalmMin}{7.5}
\newcommand{\tiltCalmMax}{9.5}
\newcommand{\tiltCalmNinetyMax}{13.6}
\newcommand{\tiltPriorDelta}{1.3}
\newcommand{\tiltCable}{47}

\newcommand{\fdCalmWins}{5}
\newcommand{\fdCalmWinsOri}{6}
\newcommand{\fdCurveCount}{6}
\newcommand{\busCentroidErr}{3.6}
\newcommand{\busStaleWorst}{4.7}

\newcommand{\busPosNoise}{0.05}
\newcommand{\traceEnvs}{32}
\newcommand{\trialsPerCell}{30}
\newcommand{\envsPerTrial}{96}

%% file: fig_architecture.tex
\definecolor{OIblue}{RGB}{0,114,178}
\definecolor{OIteal}{RGB}{0,158,115}
\definecolor{OIverm}{RGB}{213,94,0}
\definecolor{OIamber}{RGB}{230,159,0}
\tikzset{
  blk/.style={draw=black!75, thick, align=center, inner sep=3pt,
              fill=white},
  sig/.style={-{Stealth[length=2mm]}, semithick},
  trn/.style={-{Stealth[length=2mm]}, semithick,
              dash pattern=on 2pt off 1pt, draw=OIamber!60!black},
  dst/.style={-{Stealth[length=2mm]}, semithick,
              dash pattern=on 2pt off 1pt on 0.1pt off 1pt, draw=OIverm},
  shr/.style={-{Stealth[length=2mm]}, semithick,
              dash pattern=on 2pt off 1pt, draw=OIblue},
  slab/.style={font=\tiny, inner sep=1pt, fill=white},
  shrlab/.style={text=OIblue},
  trnlab/.style={text=OIamber!60!black},
  dstlab/.style={text=OIverm}}
\def\hopR{0.08}
\def\busY{-0.08}   %
\def\trnY{-0.24}   %
\scalebox{0.96}{%
\begin{tikzpicture}[x=1cm, y=1cm, font=\scriptsize]

\node[blk, fill=black!4, text width=15mm] (rob) at (0.95,2.7)
  {ground robot\\ (goal source)};
\node[blk, fill=black!4, text width=21mm] (chan) at (3.50,2.7)
  {goal broadcast\\
   {\tiny noise $\sigma_p,\sigma_\psi$}\\
   {\tiny\textcolor{OIverm}{$+$ rate / latency / dropout /}}\\
   {\tiny\textcolor{OIverm}{outliers (evaluation only)}}};
\node[blk, fill=OIblue!10, text width=27mm] (est) at (6.45,2.7)
  {per-UAV goal estimation {\tiny\textcolor{black!55}{(\ref{sec:estimation})}}\\
   $\alpha$-$\beta$ correct $+$\\ coordinated-turn predict\\
   {\tiny outputs $\hat\chi_t^g,\{\hat\delta^g_{t,k}\}_{k=1}^{K}$}};
\node[blk, fill=OIblue!16, text width=20mm] (pol) at (9.6,2.7)
  {per-UAV actors {\tiny\textcolor{black!55}{(\ref{sec:policy})}}\\
   $\pi_\theta(a_t^i\,|\,o_t^i)$, $\times 3$\\
   {\tiny parameter-shared $\theta$}};
\node[draw=black!75, thick, circle, inner sep=1.4pt, fill=white]
  (sum) at (11.36,2.7) {$+$};
\node[blk, fill=black!6, text width=15mm] (casc) at (12.81,2.7)
  {low-level\\ controller};
\node[blk, fill=black!8, text width=16mm, minimum height=11mm]
  (plant) at (15.31,2.7) {3 UAVs + arms\\ + basket};

\node[blk, draw=OIverm, fill=OIverm!8, text width=22mm,
      minimum height=6mm] (dist) at (15.31,4.0)
  {moving CoM $c_t$ (OU)\quad wind forces};
\draw[dst] (dist.south) -- (plant.north);

\draw[sig] (plant.east) -- (16.8,2.7) -- (16.8,4.65) -- (9.6,4.65)
  node[slab, pos=0.55, above=1pt]
  {$x^i_t$ includes UAV index, basket and UAV states, arm-joint positions and rates, and attachment-point error $e^i_t$}
  -- (pol.north|-{0,4.65}) -- (pol.north);

\draw[sig] (rob) -- (chan) node[slab, midway, above=2pt] {$g_t$};
\draw[sig] (chan) -- (est) node[slab, midway, above=2pt] {$y$};
\draw[sig] (est) -- (pol);
\draw[sig] (pol) -- (sum) node[slab, midway, above=2pt] {$a^i_t$};
\draw[sig] (sum) -- (casc);
\draw[sig] (casc.east|-{0,2.85}) -- (plant.west|-{0,2.85});
\draw[sig] (plant.west|-{0,2.5}) -- (casc.east|-{0,2.5})
  node[slab, midway, below=1pt, fill=none, align=center]
  {UAV\\state};

\node[blk, fill=OIteal!10, text width=40mm, minimum height=11mm]
  (ff) at (7.7,0.95)
  {anticipatory formation-yaw $+$\\
   load-share $+$ wrist-yaw feedforward {\tiny\textcolor{black!55}{(\ref{sec:ff})}}};
\node[blk, fill=OIteal!5, text width=21mm] (eso) at (13,0.95)
  {per-UAV estimators {\tiny\textcolor{black!55}{(\ref{sec:estimation})}}\\
   translational ESO $\hat d^{\,i}$\\ rotational ESO $z^i_2$};
\node[blk, draw=OIblue, densely dashed, fill=OIblue!6, text width=40mm,
      minimum height=9mm] (team) at (2.35,0.95)
  {team-bus signals {\tiny\textcolor{black!55}{(\ref{sec:estimation}, \ref{sec:ff}, \ref{sec:info})}}\\
   basket state $\omega_B$, $a_{B,z}$\\
   offsets $\hat\rho_i$, mean $\bar z_2$\\
   mass $\hat m_B^{\rm inst}\!\xrightarrow{t_m \text{ filter}} \!\hat m_B$\\
   {\tiny\textcolor{OIverm}{$+$ team-bus defects (evaluation only)}}};

\draw[sig] (est.south) -- (est.south|-ff.north)
  node[slab, pos=0.45, right=1pt] {$\hat\omega,\ \dot{\hat\omega}$};
\draw[sig] (plant.south) -- (15.31,0.95) -- (eso.east)
  node[slab, pos=0.2, above=2pt, xshift=0.1pt, fill=none, align=center] {state\\ $+$ \\thrust};
\draw[sig] (eso.west) -- (10.3,0.95) -- ({10.3,0}|-pol.south)
  node[slab, pos=-0.15, right=3pt, fill=none] {$\hat d^{\,i}\!,z^i_2\!\in o^i_t$};
\draw[shr] ($(plant.south)+(0.45,0)$) -- node[slab, shrlab, pos=0.78, right=1pt, fill=none, align=left]
  {UAV positions,\\ basket state,\\ $\sum_i\hat f_{i,z}$} (15.76,\busY)
  -- (2.35,\busY)
  -- (team.south);
\draw[shr] (eso.south) -- node[slab, shrlab, pos=0.46, right=2pt, fill=none] {$z^i_2$} (13,\busY);
\draw[shr] (12.55,\busY) -- node[slab, shrlab, pos=0.5, left=1pt, fill=none] {$\omega_B$} ({12.55,0}|-eso.south);
\draw[shr] ($(team.east)+(0,0.30)$)
  -- node[slab, shrlab, pos=0.55, above=1pt, fill=none] {$\hat\rho_i,\,\omega_{B,z}$} ($(ff.west)+(0,0.30)$);
\draw[shr] ($(team.east)-(0,0.30)$)
  -- node[slab, shrlab, pos=0.55, above=1pt, fill=none] {$\bar z_2$} ($(ff.west)-(0,0.30)$);
\draw[shr] (team.west) -- (-0.35,0.95) -- (-0.35,3.7)
  -- node[slab, shrlab, pos=0.42, above=1pt] {$\hat m_B\!\in o^i_t$} (9.0,3.7)
  -- ($(pol.north)-(0.6,0)$);
\draw[sig] (casc.south|-eso.north) -- (casc.south)
  node[slab, midway, right=1pt] {$-\hat d^{\,i}$};
\draw[sig] (9.83,1.3) -- ({10.3-\hopR},1.3)          %
  arc[start angle=180,end angle=0,radius=\hopR cm]
  -- (11.36,1.3) -- (sum.south);

\node[draw=black!55, densely dashed, fill=OIamber!7, rounded corners=3pt,
      inner sep=0pt, minimum width=14.9cm, minimum height=2.5cm,
      anchor=north west] at (2.0,-0.3) {};
\node[font=\tiny\itshape, anchor=north west, text=black!60]
  at (2.15,-0.38) {centralized training only (removed at deployment)};

\node[blk, text width=32mm, minimum height=12mm] (dr) at (4.3,-1.7)
  {domain randomization\\ wind $+$ moving CoM $+$ basket mass\\ 
   $+$  initial conditions $+$ goal motion};
\node[blk, text width=34mm, minimum height=12mm] (crit) at (8.6,-1.7)
  {disturbance-conditioned critic {\tiny\textcolor{black!55}{(\ref{sec:critic})}}\\
   return quantiles\\ lowest-quantile baseline};
\node[blk, text width=21mm, minimum height=12mm] (ppo) at (13.4,-1.7)
  {MAPPO update};
\node[blk, text width=15mm, minimum height=9mm] (rew) at (15.85,-1.7)
  {team reward $R_t$ {\tiny\textcolor{black!55}{(\ref{sec:critic})}}};
\draw[trn] (rew.west) -- (ppo.east);
\draw[trn] (dist.east) -- ({16.8-\hopR},4.0)         %
  arc[start angle=180,end angle=0,radius=\hopR cm]
  -- (17.1,4.0)
  -- (17.1,-1.7) node[slab, trnlab, pos=0.45, left=1pt, fill=none] {true wind}
  -- (rew.east);

\draw[trn] (dr.south) -- (4.3,-2.62)
  -- node[slab, trnlab, pos=0.85, above=1pt, fill=none] {mass, CoM, wind draws}
  (17.55,-2.62)
  -- (17.55,2.35)
  -- ({17.1+\hopR},2.35)                                %
  arc[start angle=0,end angle=180,radius=\hopR cm]
  -- (plant.east|-{0,2.35});
\draw[trn] (14.7,2.14) -- (14.7,{0.95+\hopR})        %
  arc[start angle=90,end angle=-90,radius=\hopR cm]
  -- (14.7,{\busY+\hopR})                              %
  arc[start angle=90,end angle=-90,radius=\hopR cm]
  -- (14.7,-0.75)
  -- (8.6,-0.75) node[slab, trnlab, pos=0.5, above=1pt, fill=none] {$s_t$, privileged $\eta_t$}
  -- (crit.north);
\draw[trn] (crit.east) -- (ppo.west)
  node[slab, trnlab, midway, above=1pt, fill=none] {estimated advantage};
\draw[trn] (ppo.north) -- (13.4,{-0.75-\hopR})       %
  arc[start angle=-90,end angle=90,radius=\hopR cm]
  -- (13.4,\trnY) -- (10.0,\trnY)
  -- (10.0,{\busY-\hopR})                              %
  arc[start angle=-90,end angle=90,radius=\hopR cm]
  -- (10.0,{1.30-\hopR})                                %
  arc[start angle=-90,end angle=90,radius=\hopR cm]
  -- ({10.0,0}|-pol.south)                              %
  node[slab, trnlab, pos=0.78, left=1pt, fill=none] {$\nabla_\theta$};

\draw[sig] (2.3,-3.1) -- (2.8,-3.1);
\node[anchor=west, font=\tiny] at (2.85,-3.1) {execution signal};
\draw[trn] (5.6,-3.1) -- (6.1,-3.1);
\node[anchor=west, font=\tiny, trnlab] at (6.15,-3.1) {training only};
\draw[dst] (8.5,-3.1) -- (9.0,-3.1);
\node[anchor=west, font=\tiny, dstlab] at (9.05,-3.1)
  {unmeasured physical disturbance};
\draw[shr] (12.4,-3.1) -- (12.9,-3.1);
\node[anchor=west, font=\tiny, shrlab] at (12.95,-3.1)
  {team-bus signal};
\end{tikzpicture}}

%% file: table_main_body.tex
CALM (ours) & \textbf{8.5$\pm$0.4/3.0$\pm$0.1} & \textbf{8.3$\pm$0.4/4.4$\pm$0.2} & \textbf{8.3$\pm$0.4/3.8$\pm$0.2} & \textbf{6.9$\pm$0.3/3.4$\pm$0.2} & \textbf{6.5$\pm$0.3/2.4$\pm$0.1} & \textbf{7.6$\pm$0.4/2.8$\pm$0.1} \\
Dec-MARL (arm) & 8.9$\pm$0.5/4.0$\pm$0.2 & 13.4$\pm$0.8/6.1$\pm$0.4 & 10.9$\pm$0.7/4.4$\pm$0.3 & 12.0$\pm$0.8/5.2$\pm$0.3 & 11.3$\pm$0.6/4.6$\pm$0.3 & 9.2$\pm$0.5/3.7$\pm$0.2 \\
Dec-MARL (cable) & 19.2$\pm$0.7/19.5$\pm$0.4 & 18.1$\pm$0.6/123.2$\pm$1.2 & 19.6$\pm$0.6/26.5$\pm$0.7 & 18.7$\pm$0.6/19.4$\pm$0.3 & 18.8$\pm$0.6/19.5$\pm$0.3 & 18.7$\pm$0.6/19.4$\pm$0.4 \\
Grasp allocation & 57.0$\pm$1.7/75.3$\pm$3.3 & 51.5$\pm$2.6/75.2$\pm$4.2 & 52.5$\pm$2.5/76.9$\pm$5.0 & 52.9$\pm$2.1/75.3$\pm$4.8 & 52.2$\pm$1.8/73.3$\pm$3.9 & 51.7$\pm$2.9/74.1$\pm$5.4 \\
RotorTM & 76.8$\pm$1.0/91.1$\pm$2.9 & 72.6$\pm$1.2/92.1$\pm$5.3 & 73.9$\pm$0.9/91.9$\pm$4.0 & 72.9$\pm$1.6/90.3$\pm$4.2 & 72.6$\pm$1.2/92.2$\pm$3.4 & 72.8$\pm$1.2/92.1$\pm$4.6 \\
\bottomrule

%% file: table_ablation_body.tex
all seven joints (passive arm) & 2.93 [2.29, 3.49] & 5.52 [4.28, 6.80] \\
wrist roll and pitch & 2.64 [2.33, 3.44] & 5.35 [4.29, 7.31] \\
shoulder and elbow (wrist active) & 1.49 [1.28, 1.75] & 1.97 [1.58, 2.46] \\
shoulder roll and pitch & 1.33 [1.20, 1.55] & 1.55 [1.36, 1.90] \\
wrist yaw & 1.17 [1.05, 1.32] & 1.15 [0.98, 1.31] \\
shoulder yaw & 1.07 [0.91, 1.31] & 1.19 [0.99, 1.41] \\
elbow & 0.92 [0.83, 1.08] & 1.01 [0.85, 1.14] \\
\bottomrule